\documentclass[letterpaper,journal]{IEEEtran}

\usepackage{enumitem}
\usepackage{amsmath}
\usepackage{amssymb}
\usepackage{amsfonts}
\usepackage{mathtools}
\usepackage{bm}
\usepackage{booktabs}
\usepackage{graphicx}
\usepackage[caption=false,font=footnotesize]{subfig}
\usepackage{color}
\usepackage{url}

\usepackage{algorithm}
\usepackage{algpseudocode}

\graphicspath{{./}{../}{figures/}{./figures/}}

\newtheorem{proposition}{Proposition}
\newtheorem{corollary}{Corollary}

\begin{document}

\title{Dream2Reward: Transition-Alignment Reward Models from Positive Demonstrations for Robotic Manipulation}

\author{Haoyu~Zhang,
        Zecui~Zeng,
        Bin~Wang,
        Lusong~Li,
        Liang~Lin,~\IEEEmembership{Fellow,~IEEE},
        and~Long~Cheng,~\IEEEmembership{Fellow,~IEEE}%
\thanks{Corresponding authors: Long Cheng; Zecui Zeng.}%
\thanks{H. Zhang and L. Cheng are with the School of Artificial Intelligence,
University of Chinese Academy of Sciences, Beijing 100049, China. They are
also with the State Key Laboratory of Multimodal Artificial Intelligence
Systems, Institute of Automation, Chinese Academy of Sciences, Beijing
100190, China (e-mail: long.cheng@ia.ac.cn).}%
\thanks{L. Li, B. Wang, Z. Zeng, and L. Lin are with JD Explore Academy,
China (e-mail: zengzecui123@gmail.com).}%
}


\markboth{IEEE Transactions on Pattern Analysis and Machine Intelligence}%
{Zhang \MakeLowercase{\textit{et al.}}:
Dream2Reward: Transition-Alignment Reward Models}

\maketitle

\begin{abstract}
Learning robotic policies requires dense rewards that remain informative when behavior departs from successful demonstrations. Progress-based rewards estimate how far an observation has advanced along a nominal successful trajectory, but may remain high after an incorrect transition. We introduce \textbf{Dream2Reward}, which learns a language-conditioned \emph{successful latent transition field} from positive demonstrations. Given the visual history up to a transition start, the model predicts the latent displacement associated with successful execution and scores the observed displacement through signed directional and symmetric magnitude agreement. This transition-level comparison penalizes wrong-direction, overshooting, and stagnant motion even when the resulting observation appears to show progress. Dream2Reward requires no failure annotations, progress labels, or synthetic negatives, and produces a dense causal reward. Across mechanism diagnostics and shared-trajectory evaluations, it provides stronger success--failure separation and more informative feedback on low-quality behavior than progress-based alternatives. Across online and offline policy learning, the same frozen reward model reduces reward hacking and supports stronger downstream performance, including in real-robot manipulation. These results show that comparing realized motion with predicted successful change provides an effective way to convert positive demonstrations into dense rewards for robot learning.
\end{abstract}


\begin{figure*}[t]
\centering
\includegraphics[width=1.0\textwidth]{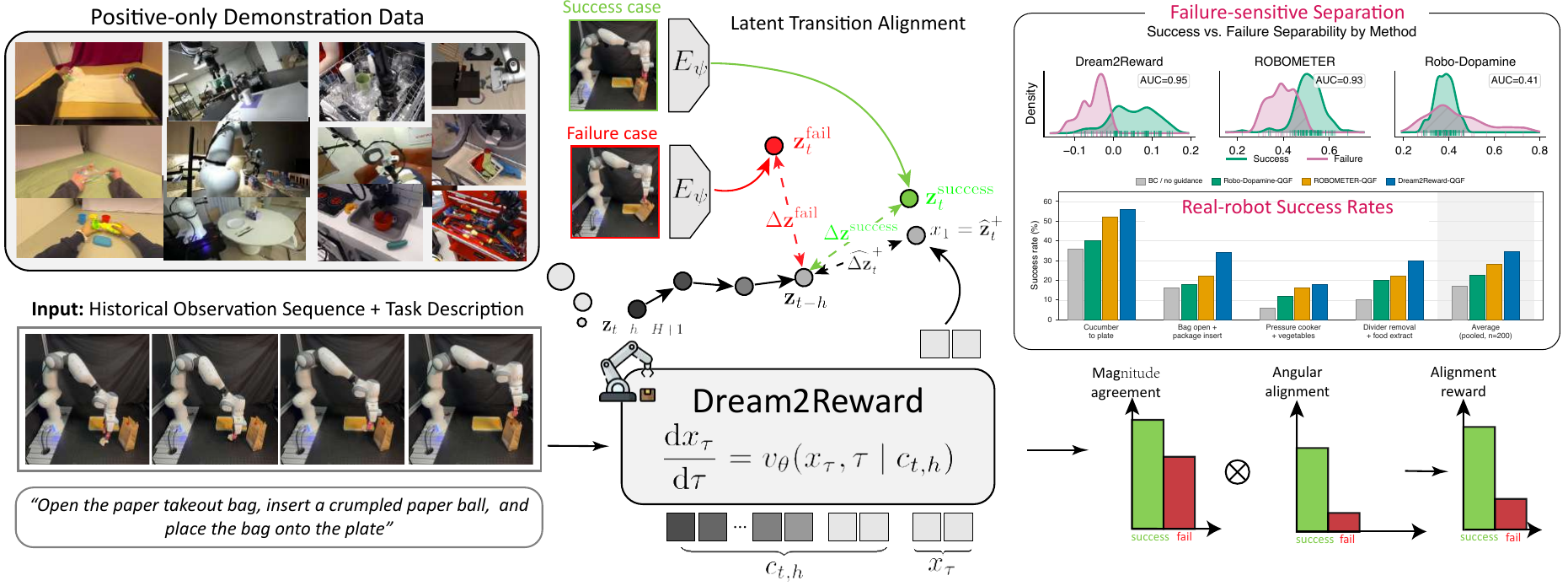}
\caption{Illustration of the Dream2Reward framework. From a language-conditioned visual history, a latent transition model trained only on positive demonstration data predicts the successful latent displacement $\widehat{\Delta z}^{+}_t$. At inference, a candidate transition is encoded into its own displacement $\Delta z_t^{\mathrm{cand}}$, and the reward is the \emph{angular alignment} and \emph{magnitude agreement} between the two: transitions along the success direction with an appropriate step size score high, whereas wrong-direction, overshooting, or stagnant transitions score low.}
\label{fig:method_overview}
\end{figure*}

\section{Introduction}

\IEEEPARstart{R}{eward} design remains a central bottleneck for applying reinforcement
learning (RL) to robotic manipulation, especially in contact-rich,
partially observed settings where sparse task success provides limited
feedback~\cite{kober2013reinforcement}. Large cross-embodiment robot
datasets and generalist vision-language robot policies have broadened
the manipulation skills that can be learned from diverse
demonstrations, reflecting a broader effort to transfer knowledge
across tasks and domains~\cite{openx,zhu2023transfer}. However,
improving a deployed policy through RL still requires dense, accurate,
and task-aligned reward signals.
The supervision available for this stage is inherently asymmetric. Successful demonstrations are increasingly abundant, while deployment-time interaction continually produces \emph{off-demonstration} transitions whose outcomes are not accompanied by dense reward annotations. A sparse success bit identifies completed episodes but gives little guidance before completion or after a mistake. Handcrafted shaping can fill this gap for one task, yet it requires access to task state and repeated engineering as objects, scenes, and embodiments change. Recent general-purpose robotic reward models instead learn visual-language rewards from robot trajectories, commonly by estimating task progress or process quality~\cite{robo_reward2026,robo_dopamine2025,robometer2026}.


Progress provides a meaningful ordering along successful trajectories: observations closer to task completion should receive higher scores than earlier ones. Once execution departs from those trajectories, however, the scalar becomes ambiguous. A spilled liquid, a dropped object, or an item pushed out of the workspace may correspond to no clear successful stage, making its progress value difficult to define. Even when a failed observation resembles a demonstrated late-stage state, neither state similarity nor progress reveals whether the transition producing it followed the correct task dynamics. The robot may move toward the target while displacing the object in the wrong direction, remain near the goal while stalled, or pass through a plausible state before overshooting. Thus, an observation may appear advanced even when the underlying transition is inconsistent with successful execution.


Positive demonstrations provide more than a temporal ordering of task progress. In addition to indicating which observations occur earlier or later along successful execution, they contain finite-horizon transitions that specify the task-relevant direction and magnitude of local change. Existing methods primarily exploit trajectory ordering or construct additional supervision through relabeling, temporal clipping, rewinding, and trajectory comparison~\cite{robo_reward2026,robometer2026,rewind2025}. By contrast, the local transition structure already present in successful trajectories has received less attention as a direct reference for scoring realized behavior.

To turn this transition structure into a usable reward signal, a model must predict the successful local change expected from the current task context. Generative video models provide a natural mechanism for constructing such a prediction from demonstrations~\cite{viper2023,diffusion_reward2024,genreward2026}. Prior reward formulations read out video-prediction likelihood,
conditional uncertainty, or alignment with generated
futures~\cite{viper2023,diffusion_reward2024,genreward2026}. Such quantities primarily evaluate the generator or its predicted endpoint, whereas policy learning requires feedback on the motion that the robot actually executes. This distinction raises a different question: \emph{can a predicted successful transition serve as a reference for scoring the transition actually executed?}


We answer this question with \textbf{Dream2Reward}, which learns a language-conditioned \emph{successful latent transition field} from positive demonstrations. Given the task goal and visual history available at a transition start, the model predicts the latent displacement expected under successful execution (Fig.~\ref{fig:method_overview}). Dream2Reward then scores the observed transition by comparing its displacement with this predicted reference in direction and magnitude. Because the reference is constructed without access to the scored endpoint, the resulting reward is causal and can penalize wrong-direction, stalled, and excessive motion at the transition level.

We evaluate Dream2Reward at the mechanism, reward-stream, and downstream-policy levels. Our main contributions are:

\begin{enumerate}
    \item We formulate positive-only robot reward learning as \emph{successful latent transition field estimation}, using successful demonstrations to learn the task-conditioned latent change expected under successful execution.

    \item We introduce a causal \emph{transition-alignment reward} that compares observed and predicted successful transitions through directional and magnitude agreement, providing dense feedback for diverse deviations without failure annotations, progress labels, or synthetic negatives.

    \item We analyze the alignment and perturbation-stability properties of the proposed reward, and evaluate its transition-level diagnostics, per-step signal quality, and failure sensitivity. We further demonstrate its utility in online policy learning and in real-robot offline policy learning across four tabletop tasks.
\end{enumerate}

\section{Related Work}


Learning rewards from demonstrations and human feedback has long been
studied through inverse and preference-based reinforcement
learning~\cite{arora2021survey}. Dynamical-system-based learning from
demonstrations represents demonstrated behavior as a motion field that
specifies how the system should evolve from each
state~\cite{zhang2023robust,zhang2026orbital}. Dream2Reward adopts a
related dynamics-centric view in latent space, but uses the predicted
successful transition field as a reward reference rather than directly
generating robot motion. Recent robotic reward models extend
demonstration-based supervision to vision-language manipulation,
combining images, language, and robot trajectories within a multimodal
learning framework~\cite{baltrusaitis2019multimodal}.

\subsection{General-Purpose Robotic Reward Models}

Early general-purpose robotic rewards are derived from visual similarity, goal-conditioned representations, or pretrained vision-language models. LIV learns language-image goal-conditioned value functions from action-free videos~\cite{liv2023}. Pretrained VLMs have also been used as zero-shot or weakly supervised reward models, through language prompts~\cite{vlmrm2024} or in-context progress prediction over frames~\cite{gvl2024}. These approaches exploit broad semantic knowledge and provide useful state-level goal or progress signals without task-specific reward engineering.

More recent large-scale models train directly on multimodal robot data. RoboReward predicts discretized progress with counterfactual relabeling and temporal clipping~\cite{robo_reward2026}; ReWiND learns progress rewards from few demonstrations and synthesizes failures by rewinding rollouts~\cite{rewind2025}; VLAC unifies action and reward in a critic model using negative and mismatched samples~\cite{vlac2025}; Robo-Dopamine builds a process reward from step-wise progress and multi-perspective fusion~\cite{robo_dopamine2025}; and ROBOMETER combines frame-level progress with trajectory-level preference supervision~\cite{robometer2026}. These methods primarily formulate reward learning in terms of state-level progress, preference, or success discrimination, often using labels, comparisons, failures, or constructed negatives. In contrast, Dream2Reward learns a task-conditioned latent transition field from successful demonstrations alone and scores each realized transition by its alignment with the predicted successful change.

\subsection{Diffusion-Based Reward Models}

Self-supervised visual learning provides general feature representations from unlabeled images and videos~\cite{jing2021selfsupervised}, while video-prediction methods model temporal structure by forecasting future observations~\cite{oprea2022video}. Modern generative formulations include diffusion and flow matching~\cite{croitoru2023diffusion,lipman2023flowmatching}, offering another route to reward signals. VIPER uses conditional video-prediction likelihoods as action-free rewards~\cite{viper2023}; Diffusion Reward derives rewards from the negative conditional entropy of a history-conditioned video diffusion model~\cite{diffusion_reward2024}; GenReward combines video-level latent alignment with frame-level goal-reaching rewards~\cite{genreward2026}; and DIFO uses a conditional diffusion discriminator's realness score over expert and agent transitions~\cite{difo2024}. These methods derive rewards from likelihood, uncertainty, generated-video alignment, or discrimination. Dream2Reward instead uses a generative model to predict task-conditioned successful change and evaluates realized transitions against this prediction, rather than directly reading out generative confidence or fit.

\subsection{Learned Rewards as Policy Objectives}

Learned reward models have increasingly been used as optimization objectives rather than only as trajectory evaluators. Goal-conditioned representations and vision-language models provide semantic rewards for policy learning~\cite{liv2023,vlmrm2024}, while video-prediction and diffusion models derive action-free rewards from predicted futures~\cite{viper2023,diffusion_reward2024}. Other approaches learn progress, critic, or preference signals from robot demonstrations and multimodal interaction data to guide policy improvement~\cite{rewind2025,vlac2025,robo_dopamine2025,robometer2026}. Across these formulations, the optimization signal is typically based on state-level progress, generative fit, or predicted success and preference. Dream2Reward instead provides a transition-level objective by comparing realized latent change with the successful change predicted from positive demonstrations.

\section{Rewards from Latent Transition Alignment}

We consider goal-conditioned manipulation as a partially observed
Markov decision process (POMDP), a standard abstraction for vision-based
robot reinforcement learning~\cite{kober2013reinforcement}. Let $s_t$ denote the underlying environment state, $o_t$ the visual observation, $a_t$ the robot action, $g$ the language-specified goal, and $\gamma$ the discount factor. Because $s_t$ is not directly observed, a policy $\pi(a_t \mid o_{0:t},g)$ acts on the observation history and maximizes the expected discounted return $\mathbb{E}_{\xi \sim \pi}[\sum_t \gamma^t r(s_t,a_t,s_{t+1};g)]$. To score the transition from $t-h$ to $t$, Dream2Reward uses a visual context of length $H$, denoted by $o_{t-h-H+1:t-h}$, together with the language goal $g$. This gives the history-dependent reward $r(o_{t-h-H+1:t-h},o_t,g)$.

Dream2Reward represents successful behavior as a task-conditioned \emph{successful latent transition field}, as illustrated in Fig.~\ref{fig:method_overview}. For a candidate transition from $t-h$ to $t$, the predictor conditions on the goal $g$ and the visual history $o_{t-h-H+1:t-h}$ to predict the successful endpoint latent $\widehat{z}^{+}_t$. Referencing this prediction to the encoded start latent $z_{t-h}$ gives the successful displacement
$\widehat{\Delta z}^{+}_t=\widehat{z}^{+}_t-z_{t-h}$. The observed endpoint $o_t$ is encoded separately as $z_t$, yielding the candidate displacement
$\Delta z_t^{\mathrm{cand}}=z_t-z_{t-h}$. The two displacements therefore share the same start point and temporal horizon.

The reward evaluates how well the candidate displacement agrees with the predicted successful displacement in direction and magnitude. Directional agreement distinguishes motion that follows the successful transition from opposing motion, while magnitude agreement captures insufficient or excessive change. The successful reference is predicted using only information available at the transition start, while the endpoint observation is used to construct the realized displacement. The resulting reward is therefore causal. Rewards are computed at every observation over overlapping $h$-step transitions, using the preceding $H$ frames as temporal context.

Let $\mathcal{D}^{+}=\{(\xi_i,g_i)\}_{i=1}^{N}$ denote a corpus of positive manipulation trajectories, where $\xi_i=(o_0^i,\ldots,o_{T_i}^i)$ and $g_i$ is the corresponding language goal. The corpus includes both success-labeled and demonstration-style trajectories, and contains no failure annotations, negative demonstrations, or synthetic negatives.

We encode each observation as $z_t=E_\psi(o_t)$ and construct the conditioning context
$c_t=C(z_{t-h-H+1:t-h},g)$ from the language goal and the $H$ frames preceding the transition. The language identifies the intended task, while the visual history captures recent motion and contact context that may disambiguate visually similar transition starts. Throughout reward-model training, the visual encoder $E_\psi$, the language encoder, and the conditioning module $C$ remain frozen.

For a fixed context $c_t$, a positive demonstration provides the target endpoint latent $z_t^{+}$ and the corresponding successful displacement
$\Delta z_t^{+}=z_t^{+}-z_{t-h}$. The successful latent transition field predicts the endpoint of this transition:
\begin{equation}
    F_\theta(c_t)\rightarrow\widehat{z}^{+}_{t},
    \qquad
    \widehat{\Delta z}^{+}_{t}
    =
    \widehat{z}^{+}_{t}-z_{t-h}.
    \label{eq:field_method}
\end{equation}
The next section describes how this predictor is learned.

\subsection{Learning the Successful Latent Transition}
\label{sec:learn_transition}

Dream2Reward does not regress the transition directly. Instead it predicts the successful \emph{endpoint latent state} with a conditional flow-matching model~\cite{lipman2023flowmatching} and obtains the transition by subtracting the observed start latent. Modeling the endpoint lets the field reuse the generative backbone of the visual representation model, keeps supervision in the encoder's native latent space, and defines every displacement relative to the actual start of the candidate being scored.

We index each reward by the endpoint of the transition being scored and train the field over multiple prediction horizons \(\kappa\in\mathcal{K}\). Given the frames preceding the transition and the offset code \(\kappa\), we form the past context and the target endpoint state
\begin{equation}
    c_{t,\kappa}
    =
    C\!\left(E_\psi(o_{t-\kappa-H+1:\,t-\kappa}), g, \kappa\right),
    \qquad
    z_t^{+}
    =
    E_\psi(o_t),
    \label{eq:flow_context_target}
\end{equation}
so the context contains only frames up to the transition start \(t-\kappa\), and the supervised endpoint \(z_t^{+}\) is the successful latent \(\kappa\) steps later. The corresponding successful displacement is
$\Delta z^{+}_{t,\kappa}=z_t^{+}-z_{t-\kappa}$.
Multi-horizon training exposes the field to successful motion at several temporal scales. At deployment, a single horizon $h$ is used for scoring, so all rewards compare displacements over a consistent physical duration; the remaining offsets serve only to enrich the training objective.

We train a conditional flow-matching velocity field that transports Gaussian noise to the successful endpoint latent using the constant endpoint-directed velocity target of rectified flow~\cite{liu2023rectifiedflow}. Let $\epsilon\sim\mathcal{N}(0,I)$ and $\tau\sim\mathcal{U}(0,1)$ denote the flow-matching time, distinct from the trajectory $\xi$. We define the interpolation
\begin{equation}
    x_\tau
    =
    (1-\tau)\epsilon+\tau z_t^{+}.
\end{equation}
The corresponding flow-matching objective is
\begin{equation}
    \mathcal{L}_{\mathrm{FM}}
    =
    \mathbb{E}
    \left[
    \left\|
    v_\theta(x_\tau,\tau\mid c_{t,\kappa})
    -
    \left(z_t^{+}-\epsilon\right)
    \right\|_2^2
    \right],
    \label{eq:flow_matching_loss}
\end{equation}
where the expectation is taken over $(o_{t-\kappa-H+1:t},g)\sim\mathcal{D}^{+}$, $\kappa\sim\mathcal{U}(\mathcal{K})$, $\tau$, and $\epsilon$. Conditioning on the visual history, language goal, and temporal offset makes the predicted endpoint specific to both the task and the transition start.

At inference, using the fixed deployment horizon, we write
$c_t=c_{t,h}$ and integrate the learned velocity field from Gaussian noise:
\begin{equation}
    \frac{d x_\tau}{d\tau}
    =
    v_\theta(x_\tau,\tau\mid c_t),
    \quad
    x_0\sim\mathcal{N}(0,I),
    \quad
    \widehat{\Delta z}^{+}_{t}
    =
    x_1-z_{t-h}.
    \label{eq:flow_generation}
\end{equation}
We use EMA weights, a fixed global initial-noise tensor, and deterministic 10-step ODE integration to produce a repeatable endpoint prediction for each context.

\subsection{Transition-Alignment Reward}
\label{sec:alignment_reward}

After training, Dream2Reward scores a candidate transition by comparing its latent displacement against the predicted successful displacement.
The candidate is the transition ending at \(t\) over the inference horizon \(h\). Concretely, given the past context \(c_t\) and the predicted successful displacement \(\widehat{\Delta z}^{+}_t\) from \eqref{eq:flow_generation}, we form the observed candidate displacement
\begin{equation}
    \Delta z_{t}^{\mathrm{cand}}
    =
    z_{t} - z_{t-h}.
    \label{eq:cand_disp}
\end{equation}
We measure the candidate's direction and magnitude relative to \(\widehat{\Delta z}^{+}_t\).

\subsubsection{Angular alignment}
The candidate transition should point in the same latent direction as successful transitions. We use the signed cosine similarity
\begin{equation}
    a_{t}
    =
    \frac{
        \langle \Delta z_{t}^{\mathrm{cand}},\, \widehat{\Delta z}^{+}_t \rangle
    }{
        \|\Delta z_{t}^{\mathrm{cand}}\|_2\,\|\widehat{\Delta z}^{+}_t\|_2 + \epsilon_{\mathrm{num}}
    }
    \;\in[-1,1].
    \label{eq:cosine}
\end{equation}

We use the signed cosine, \(r_{t}^{\mathrm{ang}}=a_{t}\), rather than rectifying it. Alignment near one denotes motion along the predicted success direction, orthogonal motion contributes little directional evidence, and opposing motion receives a negative score that can lower the value of an incorrect transition.

\subsubsection{Magnitude agreement}
A transition should also move with an appropriate latent step size. We compare displacement norms through their ratio and apply a symmetric log-ratio score,
\begin{equation}
\begin{aligned}
    \rho_{t}
    &=
    \frac{\|\widehat{\Delta z}^{+}_t\|_2 + \epsilon_{\mathrm{num}}}
         {\|\Delta z_{t}^{\mathrm{cand}}\|_2 + \epsilon_{\mathrm{num}}}, \\
    r_{t}^{\mathrm{mag}}
    &=
    \exp\!\left(
    -\frac{|\log \rho_{t}|}{\tau_m}
    \right),
\end{aligned}
\label{eq:r_mag}
\end{equation}
where we set \(\tau_m=1.0\). A raw norm difference would depend strongly on latent scale and would treat proportional deviations differently at different motion amplitudes. The absolute log-ratio instead gives the same penalty when candidate and reference magnitudes are exchanged. It therefore treats proportional under-motion and over-motion symmetrically, covering no motion or slipping as well as overshooting and abrupt latent jumps.

\subsubsection{Per-step raw reward}
We first combine direction and magnitude multiplicatively, then subtract an explicit magnitude-shortfall penalty and clamp,
\begin{align}
    s_{t} &= r_{t}^{\mathrm{ang}}\cdot r_{t}^{\mathrm{mag}}, \label{eq:per_cand}\\
    r_t^{\mathrm{raw}}
    &=
    \mathrm{clip}\!\left(\,s_{t} - \lambda_{\mathrm{pen}}\,(1-r_{t}^{\mathrm{mag}}),\;-1,\;+1\right),
    \label{eq:final_score}
\end{align}
with \(\lambda_{\mathrm{pen}} = 0.3\). Multiplication implements a conjunction: a candidate cannot obtain a high core score by matching direction while moving at an implausible scale, or by matching the expected scale while moving elsewhere. In contrast, an additive combination can allow one strong component to compensate for the other. The shortfall term strengthens the penalty for severe magnitude mismatch, and clamping bounds the raw model output in \([-1,1]\). We retain \(r_t^{\mathrm{ang}}\) and \(r_t^{\mathrm{mag}}\) separately for mechanism diagnostics.

The three stages of the score correspond to distinct deviations from successful execution. If the robot reverses, drifts laterally, or moves an object toward the wrong target, the signed angular term falls even when the candidate norm is typical. If the scene is nearly static because contact is missed or the object is stuck, the candidate norm is too small; if an object is dropped, flung, or pushed past its intended configuration, the norm can be too large. In both cases the magnitude term falls even if the displacement has a positive projection on the successful direction. A transition that is moderately wrong in both respects is further suppressed by the product.

The context-conditioned successful displacement provides a common representation for these events without category-specific supervision. The same equations therefore apply across tasks and failure modes, including local recovery: when a later candidate again follows the predicted successful field, its reward can rise without waiting for an episode-level outcome.

\subsection{Reward Normalization for Policy Learning}
\label{sec:calibration}

Let $q_t$ denote a reward model's native causal output, with $q_t=r_t^{\mathrm{raw}}$ for Dream2Reward and $q_t=p_t$ for a progress-based baseline. Before policy or critic learning, each reward stream is standardized using task-specific statistics:
\begin{equation}
    r_t=\frac{q_t-\mu}{\sigma},
    \label{eq:rl_reward}
\end{equation}
where $\mu$ and $\sigma$ are estimated once from data available before downstream optimization and then held fixed. The standardized reward $r_t$ is used directly by online RL or as the transition-level supervision for offline critic learning.

This affine transformation preserves the temporal ordering and local extrema of each native reward stream while placing different methods on comparable offsets and scales. Fixing the normalization statistics also keeps the reward mapping unchanged throughout downstream optimization. The normalization data used in each experimental setting are specified in the corresponding experimental protocol.

Algorithm~\ref{alg:Dream2Reward} summarizes the overall procedure. Dream2Reward trains the latent transition field using only positive transitions from $\mathcal{D}^{+}$, and scores candidate transitions through latent transition alignment.

\begin{algorithm}[t]
\caption{Dream2Reward: Training and Reward Computation}
\label{alg:Dream2Reward}
\begin{algorithmic}[1]
\Require Positive demonstrations \(\mathcal{D}^{+}\), frozen visual and language encoders, non-trainable context construction \(C\), trainable flow field \(v_\theta\), horizon set \(\mathcal{K}\), inference horizon \(h\)

\Statex \textbf{Training}
\For{each training iteration}
    \State Sample \((\xi,g)\sim\mathcal{D}^{+}\), endpoint \(t\), offset \(\kappa\sim\mathcal{U}(\mathcal{K})\)
    \State Encode past frames and endpoint with \(E_\psi\)
    \State Build past context \(c_{t,\kappa}\) via \eqref{eq:flow_context_target}
    \State Fit \(v_\theta\) to the endpoint latent \(z_t^{+}\) via \eqref{eq:flow_matching_loss}
\EndFor

\Statex \textbf{Reward Computation}
\For{each video \(v\) and step \(t\)}
    \State Build past context \(c_t\) at horizon \(\kappa=h\)
    \State Predict endpoint \(x_1\) from the fixed global noise \(x_0\) via \eqref{eq:flow_generation}, set \(\widehat{\Delta z}^{+}_t=x_1-z_{t-h}\)
    \State Form candidate displacement \(\Delta z_t^{\mathrm{cand}}\) using Eq.~\eqref{eq:cand_disp}
    \State Compute \(r_t^{\mathrm{raw}}\) using Eqs.~\eqref{eq:cosine}--\eqref{eq:final_score}
\EndFor
\State \Return causal per-step raw reward sequence $r_t^{\mathrm{raw}}$
\end{algorithmic}
\end{algorithm}

\subsection{Implementation Details}
\label{sec:impl}

All experiments use the following fixed instantiation. The frozen visual encoder \(E_\psi\) is an RAEv2 representation autoencoder~\cite{singh2026raev2} with a DINOv3 ViT-L/16 backbone~\cite{dinov3}; it maps each \(256{\times}256\) frame to a \((1024,16,16)\) latent. The frozen language encoder is Qwen3-0.6B~\cite{qwen3}. The conditioning construction \(C\) is non-trainable, and only the velocity field \(v_\theta\) is optimized with \eqref{eq:flow_matching_loss}. Inference uses EMA weights, the frozen global noise tensor, and deterministic 10-step ODE integration. We use \(\epsilon_{\mathrm{num}}=10^{-6}\), \(\tau_m=1.0\), and \(\lambda_{\mathrm{pen}}=0.3\), shared across tasks, which bounds the raw score in \([-1,1]\).

\emph{Training data and temporal grid.} We train on positive manipulation demonstrations from EgoDex~\cite{egodex} and Open X-Embodiment~\cite{openx}, sampled on a fixed \(4\,\mathrm{fps}\) grid. The history contains \(H{=}4\) frames, training uses offsets \(\mathcal{K}=\{1,2,4,8,16\}\), and deployment uses \(h{=}2\), corresponding to a \(0.5\)s candidate transition. We treat both datasets as positive corpora: EgoDex contributes egocentric \emph{human} manipulation videos. Within Open X-Embodiment, we use success-labeled trajectories
when outcome metadata are available and otherwise use
demonstration-style trajectories as positive examples.

For visualization only, reward traces are causally smoothed with an EMA coefficient of $0.35$; all RL experiments use the unsmoothed per-step reward.

\section{Properties of Transition-Alignment Rewards}
\label{sec:properties}

We now characterize the geometry and local stability of transition alignment. Write the observed latent transition as \(u=z_{t}-z_{t-h}\) and the predicted successful transition as \(v=\widehat{\Delta z}^{+}_t\). The analysis omits the numerical stabilizer \(\epsilon_{\mathrm{num}}\) and concerns nonzero transitions; the implementation retains \(\epsilon_{\mathrm{num}}=10^{-6}\). The resulting non-degenerate core score is
\begin{equation}
    r(u,v)=\cos(u,v)\,\exp\!\left(-\tfrac{1}{\tau_m}\left|\log\tfrac{\|u\|}{\|v\|}\right|\right),
    \label{eq:prop_reward}
\end{equation}
with \(\tau_m>0\) and \(u,v\neq 0\). The corollary below extends this non-degenerate score through the magnitude-shortfall penalty and clipping.

\begin{proposition}[Alignment maximization and deviation penalty]
\label{prop:align}
For any nonzero \(u,v\), the reward \eqref{eq:prop_reward} satisfies \(r(u,v)\le 1\), with equality if and only if \(u,v\) are positively colinear and \(\|u\|=\|v\|\); the angular and magnitude factors act \emph{separately}, the former maximized by colinearity and the latter independently penalizing scale mismatch. Quantitatively, an angular deviation \(\angle(u,v)\ge\alpha\) (with \(\alpha\le\pi/2\)) caps the reward at \(r(u,v)\le\cos\alpha<1\), and a log-magnitude deviation \(|\log(\|u\|/\|v\|)|\ge\delta\) caps it at \(r(u,v)\le\exp(-\delta/\tau_m)<1\); if both hold, the tighter product bound \(r(u,v)\le\cos\alpha\,\exp(-\delta/\tau_m)\) applies. For \(\alpha>\pi/2\) the direction is already wrong (\(\cos(u,v)<0\)) and \(r\le 0\), strictly below the maximum.
\end{proposition}
\noindent\emph{Proof.} Write \(r=g\,h\) with \(g=\cos(u,v)\in[-1,1]\) and \(h=\exp(-|\log(\|u\|/\|v\|)|/\tau_m)\in(0,1]\); \(g=1\) iff the directions coincide and \(h=1\) iff \(\|u\|=\|v\|\), so \(r\le 1\) with equality only when both hold. For \(\alpha\le\pi/2\) an angular deviation \(\ge\alpha\) gives \(g\le\cos\alpha\), so \(r=g\,h\le\cos\alpha\) since \(h\le 1\); a log-magnitude deviation \(\ge\delta\) gives \(h\le\exp(-\delta/\tau_m)\), so \(r\le\exp(-\delta/\tau_m)\) since \(g\le 1\); both together give the product bound. For \(\alpha>\pi/2\), \(g<0\) so \(r\le 0\).  \hfill\(\square\)

Proposition~\ref{prop:align} formalizes the conjunction implemented by the two reward factors. Achieving the maximum requires both the successful direction and the successful scale; satisfying either condition alone is insufficient. The quantitative bounds also distinguish the failure modes: a turn away from the reference is controlled by the angular factor, while stalling or overshooting increases the log-magnitude deviation. When both occur, their penalties compound through the product.

\begin{proposition}[Transition alignment resolves state-progress ambiguity]
\label{prop:ambiguity}
Consider successful and failed candidate transitions from the same past context, with identical scalar progress at their endpoints, \(p(z_t^{\mathrm{succ}})= p(z_t^{\mathrm{fail}})\), and let \(v\) be the predicted successful transition for that context. For any state-only reward that is solely a function of this scalar progress, if
\begin{equation}
    \cos(\Delta z_t^{\mathrm{succ}},v) \;>\; \cos(\Delta z_t^{\mathrm{fail}},v),
    \label{eq:prop_amb}
\end{equation}
then any such state-only progress reward, being a function of the (identical) scalar progress values alone, assigns the two transitions equal reward, whereas the alignment reward \eqref{eq:prop_reward} assigns strictly greater reward to the successful transition whenever their magnitude factors are equal.
\end{proposition}
\noindent\emph{Proof.} A progress-based reward depends on its input only through the scored state's progress value; since these coincide by assumption, it returns equal reward for both transitions irrespective of their direction. The alignment reward is strictly increasing in \(\cos(\cdot,v)\) at a fixed magnitude factor, so \eqref{eq:prop_amb} makes it strictly larger for the successful transition. The separating quantity is exactly the directional information that a scalar progress value discards. \hfill\(\square\)

Proposition~\ref{prop:ambiguity} isolates the information gained by moving from states to transitions. Two endpoints may occupy the same scalar progress level even when they were reached through different latent directions. A state-progress function must collapse these cases by construction, while the predicted vector provides an oriented local reference. Magnitude agreement then complements this directional separation when the two candidates also differ in how far they move.

\begin{proposition}[Stability under field perturbation]
\label{prop:stability}
Fix \(u\) with \(\|u\|>0\), let \(v\) be the nominal predicted successful transition, and suppose its perturbation \(\widetilde v\) satisfies \(\|\widetilde v-v\|\le\varepsilon\) with \(\|v\|,\|\widetilde v\|\ge m>0\). Then
\begin{equation}
    \big|r(u,\widetilde v)-r(u,v)\big|\;\le\;L\,\varepsilon,
    \label{eq:prop_lip}
\end{equation}
for the explicit constant \(L=\tfrac{2}{m}+\tfrac{1}{\tau_m}\tfrac{1}{m}\), which depends only on the lower norm bound \(m\) and the temperature \(\tau_m\).
\end{proposition}
\noindent\emph{Proof.} Write \(r(u,w)=g(w)h(w)\), where \(g(w)=\cos(u,w)\in[-1,1]\) and \(h(w)=\exp(-|\log(\|u\|/\|w\|)|/\tau_m)\in(0,1]\). Direct endpoint comparison gives \(\|v/\|v\|-\widetilde v/\|\widetilde v\|\|\le(2/m)\|v-\widetilde v\|\), hence \(|g(\widetilde v)-g(v)|\le(2/m)\|\widetilde v-v\|\). It also gives \(\big|\log\|v\|-\log\|\widetilde v\|\big|\le |\|v\|-\|\widetilde v\||/\min(\|v\|,\|\widetilde v\|)\le\|v-\widetilde v\|/m\). Together with \(\big||a|-|b|\big|\le|a-b|\), this inequality and the \(1/\tau_m\)-Lipschitz continuity of \(x\mapsto e^{-x/\tau_m}\) on \([0,\infty)\) yield \(|h(\widetilde v)-h(v)|\le\|\widetilde v-v\|/(\tau_m m)\). Since \(|g|,h\le1\), adding and subtracting \(g(\widetilde v)h(v)\) gives \(|r(u,\widetilde v)-r(u,v)|\le(2/m+1/(\tau_m m))\|\widetilde v-v\|\le L\varepsilon\). \hfill\(\square\)

Proposition~\ref{prop:stability} shows that, away from near-zero transition norms, bounded perturbations of the predicted transition field induce bounded changes in the reward. The sensitivity is controlled by the lower norm bound \(m\) and temperature \(\tau_m\).

\begin{corollary}[Stability of the post-processed non-degenerate score]
\label{cor:post_stability}
Under the assumptions of Proposition~\ref{prop:stability}, let
\(q(u,v)=\exp(-|\log(\|u\|/\|v\|)|/\tau_m)\) and define
\begin{equation}
 r_{\mathrm{post}}(u,v)=\mathrm{clip}\!\left(r(u,v)-\lambda_{\mathrm{pen}}(1-q(u,v)),-1,1\right).
\end{equation}
Then
\begin{equation}
 \left|r_{\mathrm{post}}(u,\widetilde v)-r_{\mathrm{post}}(u,v)\right|
 \le L_{\mathrm{post}}\varepsilon.
\end{equation}
Here \(L_{\mathrm{post}}=2/m+(1+\lambda_{\mathrm{pen}})/(\tau_m m)\).
\end{corollary}
\noindent\emph{Proof.} The clipping operator is $1$-Lipschitz. Proposition~\ref{prop:stability} gives the endpoint bound $2/m+1/(\tau_m m)$ for $r(u,\cdot)$. For $q$, the same endpoint inequality
\(
|\log\|v\|-\log\|\widetilde v\||\le\|v-\widetilde v\|/m
\)
and the $1/\tau_m$-Lipschitz continuity of the exponential give
\(|q(u,\widetilde v)-q(u,v)|\le\|\widetilde v-v\|/(\tau_m m)\).
Applying the triangle inequality before clipping gives the stated constant. \hfill\(\square\)

Thus the local perturbation bound extends from the core alignment score to its magnitude-penalized and clipped non-degenerate form.

\section{Experiments}
\label{sec:experiments}


The experiments follow a progression from validating the reward mechanism to assessing its downstream utility. Section~\ref{sec:diagnostics} first examines whether latent transition alignment distinguishes successful execution from task-mismatched and failed-episode transitions. Section~\ref{sec:fail_sensitive} then evaluates whether the resulting reward provides reliable and timely separation between successful and failed behavior, while Section~\ref{sec:ablation} identifies the contributions of direction, magnitude, and language conditioning. Section~\ref{sec:libero_online} tests whether these properties translate into improved online policy learning, and Section~\ref{sec:online_why} analyzes the temporal reward characteristics underlying the observed learning behavior. Finally, Section~\ref{sec:real_robot_qgf} evaluates whether Dream2Reward can provide effective zero-shot labels for real-robot offline RL, and Section~\ref{sec:offline_why} examines whether the resulting policy gains are supported by stronger transition-level credit assignment.

Unless otherwise specified, reward-level comparisons include Dream2Reward, ROBOMETER~\cite{robometer2026}, and Robo-Dopamine~\cite{robo_dopamine2025}, evaluated through their causal per-step outputs. All reward models remain frozen throughout evaluation and downstream policy learning. Because the applicable baselines and score-processing protocols differ across reward diagnostics, online RL, and offline RL, each subsection specifies its corresponding comparison protocol.


\subsection{Latent Transition Alignment Diagnostics}
\label{sec:diagnostics}


For the mechanism and failure-separation analyses, we construct a real-robot scoring set over the four manipulation tasks described in Sec.~\ref{sec:real_robot_qgf}. For each task, the set contains 50 successful human demonstrations and 50 rollouts from a fine-tuned policy, yielding 400 trajectories in total. Episode-level outcome annotation identifies 234 successful and 166 failed trajectories. All 200 human demonstrations are successful, whereas 34 of the 200 policy rollouts succeed and 166 fail.

We score three groups: transitions from successful demonstrations; wrong-language scores obtained by evaluating successful clips under a mismatched instruction; and transitions sampled from failed policy episodes. We aggregate transition scores within each trajectory before testing. Correct- versus wrong-language scores are paired and use a two-sided Wilcoxon signed-rank test, whereas successful- versus failed-episode scores are unpaired and use a two-sided Mann--Whitney \(U\) test.  Figure~\ref{fig:alignment_diagnostics} shows \(r_t^{\mathrm{raw}}\), \(r_t^{\mathrm{ang}}\), and \(|\log\rho_t|\). Successful transitions receive higher raw reward than either comparison group; wrong-language and failed-episode transitions also show larger magnitude mismatch and lower angular alignment.

\begin{figure}[tb]
    \centering
    \includegraphics[width=0.9\linewidth]{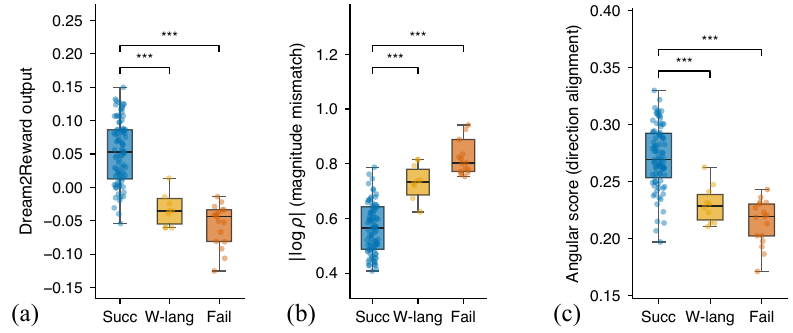}
    \caption{Latent transition alignment diagnostics for successful, wrong-language, and failed-episode transitions: (a) raw reward \(r_t^{\mathrm{raw}}\), (b) absolute log magnitude ratio \(|\log\rho_t|\), and (c) angular score \(r_t^{\mathrm{ang}}\). Significance is computed on per-trajectory aggregates: paired correct- versus wrong-language scores use a two-sided Wilcoxon signed-rank test, while successful- versus failed-episode scores use a two-sided Mann--Whitney \(U\) test; \(\ast\ast\ast\): \(p<0.001\), \(\ast\ast\): \(p<0.01\), \(\ast\): \(p<0.05\), n.s.: \(p\ge0.05\).}
\label{fig:alignment_diagnostics}
\end{figure}

Wrong-language and failed-episode groups have lower angular scores, so their realized motion points less consistently along the task-conditioned success field. Their larger absolute log-ratios also indicate less agreement in the amount of change. These component differences are consistent with the lower raw rewards observed for both groups, neither of which appears in reward-model training.

\subsection{Failure-Sensitive Reward Validation}
\label{sec:fail_sensitive}

\begin{figure}[tb]
    \centering
    \includegraphics[width=0.84\linewidth]{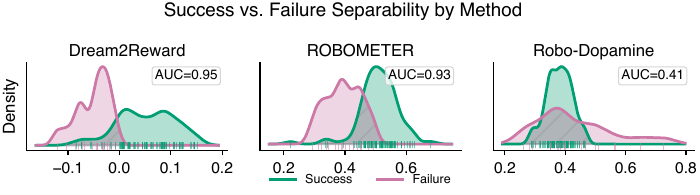}
    \caption{Failure-detection ROC-AUC on identical real-robot trajectories, using negative episode-mean native rewards.}
    \label{fig:fail_auc}
\end{figure}

On the same 400-trajectory real-robot scoring set, we evaluate episode-level success--failure separation of each native per-step reward stream. This comparison complements the component plots by measuring whether transition-level differences accumulate into a useful episode ordering and reward margin.


Failure separation is summarized by two complementary metrics:

\begin{enumerate}[
    label=(\arabic*),
    leftmargin=*,
    labelsep=0.4em,
    itemsep=0.4em,
    topsep=0.3em,
    parsep=0pt
]
    \item \emph{Failure-detection AUC.}
    Each episode is represented by the mean of its native causal reward stream, and failure-detection ROC-AUC is computed using the negative episode score. Because AUC depends only on ranking, it measures whether successful episodes tend to receive higher rewards than failed ones without being affected by reward scale (Fig.~\ref{fig:fail_auc}).

    \item \emph{Success--failure reward gap.}
    Using the evaluation-only score $\widetilde r_t$, we compute
    $\Delta\widetilde r
    =
    \mathbb{E}[\widetilde r_t\mid\text{successful episode}]
    -
    \mathbb{E}[\widetilde r_t\mid\text{failed episode}]$.
    This metric measures the separation margin after mapping each method to its own $[0,1]$ evaluation scale (Fig.~\ref{fig:fail_gap}).
\end{enumerate}

\begin{table}[tb]
\centering
\caption{Failure-sensitive reward validation on the real-robot scoring set.}
\label{tab:fail_sensitive}
\setlength{\tabcolsep}{4pt}
\begin{tabular}{lcc}
\toprule
Model & AUC $\uparrow$ & Gap $\uparrow$ \\
\midrule
Dream2Reward & 0.95 & +0.376 \\
ROBOMETER & 0.93 & +0.253 \\
Robo-Dopamine & 0.41 & $-$0.107 \\
\bottomrule
\end{tabular}
\end{table}

Table~\ref{tab:fail_sensitive} reports both metrics, and Fig.~\ref{fig:real_robot_failure_traces} shows representative per-step reward traces. On this evaluation set, Dream2Reward and ROBOMETER have similar episode-ordering AUCs ($0.95$ and $0.93$), while Dream2Reward has the larger within-method min--max-normalized success--failure gap ($+0.376$ versus $+0.253$). Robo-Dopamine yields a sub-chance AUC and negative gap on this set. The per-step traces help explain this difference in separation. Many failed rollouts reach late, high-progress-looking states before a drop or stall; the progress streams often remain high there, whereas Dream2Reward decreases when the observed motion no longer matches the predicted successful transition.

\begin{figure}[b]
    \centering
    \includegraphics[width=0.60\linewidth]{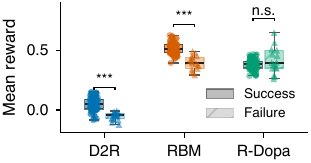}
    \caption{Evaluation-normalized success--failure gap on identical trajectories; significance uses a two-sided Mann--Whitney \(U\) test on per-trajectory scores. D2R denotes Dream2Reward, RBM denotes ROBOMETER, and R-Dopa denotes Robo-Dopamine.}
    \label{fig:fail_gap}
\end{figure}

Figure~\ref{fig:real_robot_failure_traces} shows how this separation arises over time. In the stalled-grasp and paper-bag examples, visual progress accumulates before contact or object motion ceases, but the realized displacement no longer has the expected magnitude. Plate-edge interference disrupts the expected local transition despite a plausible grasp configuration. Conversely, after an initial grasp miss, the recovery motion realigns with the predicted field and the Dream2Reward trace rises again. The reward therefore responds to local failure and recovery events rather than assigning a fixed score to the episode as a whole.

\begin{figure*}[tp]
    \centering
    \includegraphics[width=0.87\textwidth]{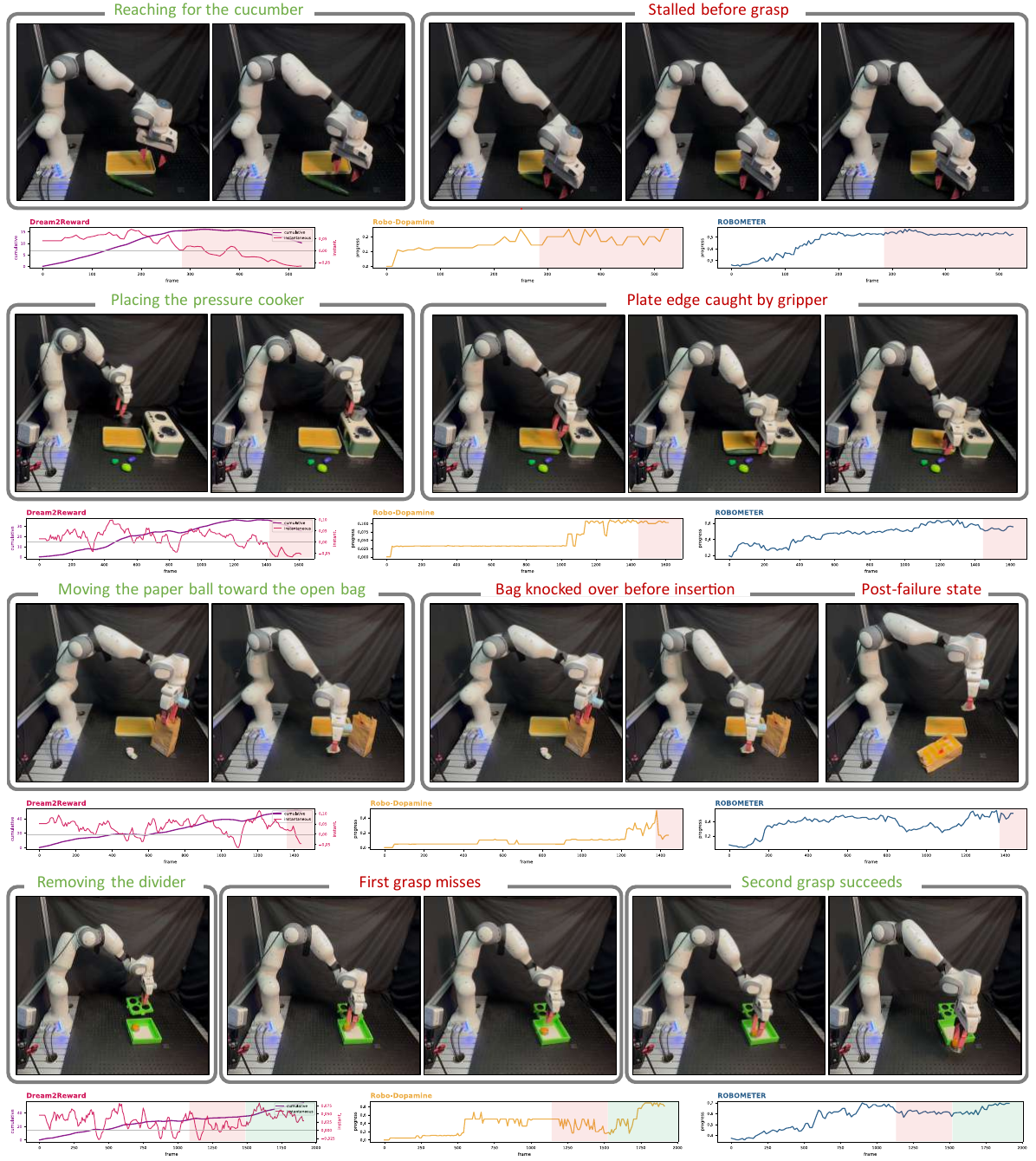}
    \caption{
    Representative real-robot failure and recovery traces.
    Each block shows task snapshots on the top and synchronized reward traces below, comparing Dream2Reward with the progress-based baselines ROBOMETER and Robo-Dopamine on the same rollout.
    Green labels indicate nominal task progress or successful recovery, red labels indicate failure events or post-failure states, and the shaded intervals are qualitative annotations for the displayed traces.
    The examples cover stalled grasping, paper-bag cleanup failure, plate-edge interference during grasping, and recovery after an initial grasp miss.
    }
    \label{fig:real_robot_failure_traces}
\end{figure*}

\subsection{Ablation Study}\label{sec:ablation}

We ablate angle only, magnitude only, and no language conditioning, removing one factor of \eqref{eq:per_cand} or the task text \(g\). The first two variants test whether either geometric component alone explains the result; the third tests whether the field can score motion reliably without knowing which task the same scene is meant to perform. 

Table~\ref{tab:ablation} reports the failure-detection AUC for each variant.
Angle-only and magnitude-only each obtain $0.92$ AUC, while their product obtains $0.95$, indicating that neither single factor matches the full readout on this evaluation. Removing language conditioning lowers AUC to $0.80$. Together with the denoising-energy control, these results associate the observed separation with task-conditioned transition geometry rather than either factor alone.

We further test whether failure separation follows from the generative predictor alone or from the transition-alignment readout. Keeping the frozen encoders, conditioning construction $C$, and trained velocity field $v_\theta$ unchanged, we replace transition alignment with denoising energy. This control achieves a pooled ROC-AUC of $0.458$ and a mean per-task AUC of $0.273$, showing that denoising energy does not recover the failure separation provided by transition alignment.

\begin{table}[tb]
\centering
\caption{Ablation results on failure-sensitive validation. All variants use the same evaluation protocol; AUC is failure-detection ROC-AUC.}
\label{tab:ablation}
\setlength{\tabcolsep}{4pt}
\begin{tabular}{lcc}
\toprule
Variant & Removed / changed component & AUC $\uparrow$ \\
\midrule
Full Dream2Reward & -- & 0.95 \\
Angle only & magnitude term & 0.92 \\
Magnitude only & angular term & 0.92 \\
No language conditioning & task text $g$ & 0.80 \\
\bottomrule
\end{tabular}
\end{table}

\subsection{Online Reinforcement Learning on LIBERO}
\label{sec:libero_online}

Online RL uses LIBERO-90 Tasks 28 (\emph{close the top drawer}) and 33 (\emph{close the microwave})~\cite{libero2023}, matching the task selection and protocol reported for ROBOMETER~\cite{robometer2026}. For each reward, the normalization statistics $(\mu,\sigma)$ are estimated once from the reward rollouts collected before SAC optimization and remain fixed throughout training. SAC is trained from scratch for 300k environment steps; evaluation uses 10 episodes every 5000 steps and reports mean and standard deviation over 5 seeds. Reward queries use only the external camera, while the policy receives DINO-v2-small~\cite{dinov2} external-image features and proprioception.


For the controlled online RL comparison, we run Dream2Reward and Robo-Dopamine under the same SAC pipeline. Both reward models receive observations from the same external camera at the same query rate, with Robo-Dopamine evaluated in its native single-view setting. We use the same policy architecture, interaction budget, and evaluation protocol for both methods. ROBOMETER is not rerun in this pipeline; its result on the same LIBERO tasks is reported from the original study~\cite{robometer2026}. For each rerun method and task, reward outputs are standardized using rollout statistics collected before policy optimization, and the resulting normalization parameters remain fixed throughout training.

Checkpoint-level analysis in Fig.~\ref{fig:reward_hacking} examines whether higher model rewards are associated with higher true task success rates.
Each checkpoint is summarized by its mean evaluation-normalized model reward and true success rate across evaluation episodes. The late-stage gap \(H_{\mathrm{late}}=\overline{\tilde r}-\overline{s}\) is computed over the final $25\%$ of evaluation checkpoints, and the top-$20\%$ statistic measures true success among the checkpoints ranked highest by model reward. The scatter in Fig.~\ref{fig:online_rl} directly shows high-reward/low-success checkpoints. Environment success is used only for evaluation, not SAC training.

\begin{figure}[tbp]
    \centering
    \subfloat[LIBERO Task 28 and Task 33 learning curves (SAC, 5 seeds).\label{fig:libero_online_rl}]{%
        \includegraphics[width=0.95\linewidth]{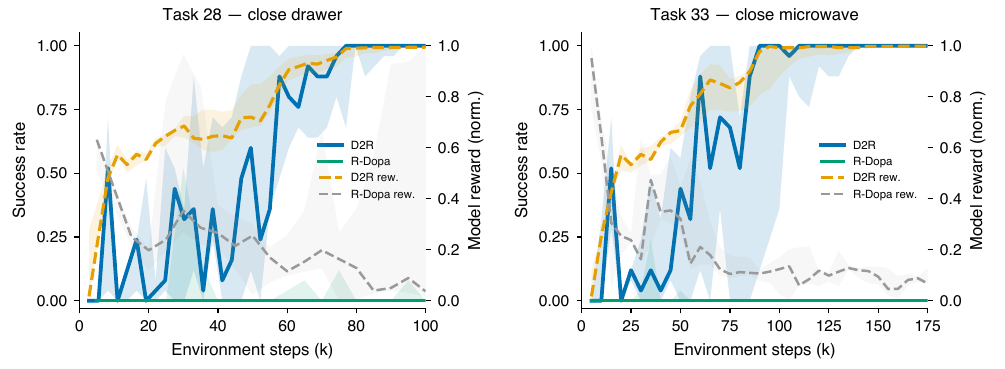}}

    \vspace{0.4em}
    \subfloat[Reward-hacking analysis on LIBERO Task 33.\label{fig:reward_hacking}]{%
        \includegraphics[width=0.95\linewidth]{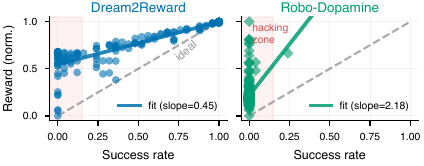}}
    \caption{Online RL comparison between Dream2Reward and Robo-Dopamine under the same SAC pipeline. (a) True task success and evaluation-normalized model reward on two manipulation tasks over five seeds; $\widetilde r_t\in[0,1]$ is used only for visualization. (b) Model reward versus true success across evaluation checkpoints. The line shows a least-squares fit, and the shaded upper-left region indicates high predicted reward despite low task success.}
    \label{fig:online_rl}
\end{figure}

Figs.~\ref{fig:libero_online_rl}--\ref{fig:reward_hacking} show that Dream2Reward drives SAC to near-full success on both tasks, whereas Robo-Dopamine solves neither. The distinction appears in both panels: Dream2Reward's model reward rises with the sparse success curve, while many Robo-Dopamine checkpoints move upward in model reward without moving toward task completion. On Task 33, Dream2Reward has a near-zero late-stage hacking gap ($|H_{\mathrm{late}}|=0.001$ versus $0.222$) and $1.00$ true success among the top-$20\%$ checkpoints ranked by model reward (versus $0.005$). Transition alignment makes a high score contingent on executing the expected local motion, reducing the opportunity to increase reward through a visually plausible but task-incorrect configuration.

\subsection{Per-Step Reward Quality for Online RL}
\label{sec:online_why}

The LIBERO results establish the downstream learning advantage, but they do not isolate the per-step reward properties that support it. We therefore evaluate reward quality independently of policy optimization, focusing on whether each reward stream provides dense, failure-sensitive, and temporally informative feedback.


For this analysis, we use the robo\_arena subset of RoboReward~\cite{robo_reward2026}, which is independent of both the 400-trajectory real-robot scoring set and the LIBERO training runs. It contains episodes from 40 tasks, episode-level quality ratings, and outputs from all three frozen reward models on identical frames. Episodes rated 4 or 5 are treated as high quality, whereas those rated 1 or 2 are treated as low quality; episodes rated 3 are excluded from group-conditioned analyses. These rollouts provide a shared basis for comparing reward-stream properties and are not used for reward-model training or LIBERO policy optimization.

To compare reward streams with different native ranges and scales, we first apply a common robust normalization. Each method's native output is centered and scaled per task using that method's median and median absolute deviation (MAD). We denote the resulting diagnostic stream by $r^{\mathrm{diag}}$, which is distinct from the mean- and standard-deviation-normalized reward $r_t$ used for policy learning in Eq.~\eqref{eq:rl_reward}. We then evaluate four complementary aspects of per-step reward quality:

\begin{enumerate}[
    label=(\arabic*),
    leftmargin=*,
    labelsep=0.4em,
    itemsep=0.4em,
    topsep=0.3em,
    parsep=0pt
]
    \item \emph{Locally flat-step fraction.}
    For $r^{\mathrm{diag}}=(r^{\mathrm{diag}}_1,\ldots,r^{\mathrm{diag}}_T)$, this metric measures the fraction of consecutive steps with negligible reward variation:
    \begin{equation}
        \mathrm{ZG}
        =
        \frac{1}{T-1}
        \sum_{t=2}^{T}
        \mathbf{1}
        \left[
        \left|
        r^{\mathrm{diag}}_t-r^{\mathrm{diag}}_{t-1}
        \right|
        \leq \epsilon_0
        \right],
        \label{eq:zerograd}
    \end{equation}
    where $\epsilon_0=0.02$ in MAD units.

    \item \emph{Signal coverage.}
    The fraction of steps whose absolute deviation from the trajectory median exceeds $0.5$ MAD. It measures how broadly non-trivial reward variation is distributed over a rollout.

    \item \emph{Low-quality negative-reward fraction.}
    The fraction of steps in low-rated episodes whose diagnostic reward falls below the task-centered neutral baseline. It measures whether low-quality behavior receives negative rather than merely small feedback.

    \item \emph{Reward--time mutual information.}
    We compute
    $I(R;T)=\sum_{i,j}p_{ij}\log\!\left(p_{ij}/(p_i p_j)\right)$
    using 10 equal-frequency reward bins and 10 uniform bins of normalized elapsed time. Lower mutual information indicates that the reward responds to executed behavior rather than primarily tracking elapsed time.
\end{enumerate}

Together, these metrics assess whether a reward stream provides broadly distributed and behavior-sensitive feedback without remaining locally flat or primarily tracking elapsed time. Table~\ref{tab:signal_quality} shows that Dream2Reward is locally flat on only $2.4\%$ of steps while maintaining a signal coverage of $0.942$. ROBOMETER is similarly non-flat but assigns below-neutral rewards to fewer steps in low-quality episodes. In contrast, Robo-Dopamine is locally flat on $57.7\%$ of steps and achieves a coverage of only $0.422$. Dream2Reward assigns below-neutral rewards to $74.5\%$ of steps in low-quality episodes, compared with $41.3\%$ for ROBOMETER and $6.1\%$ for Robo-Dopamine. These results indicate that Dream2Reward provides broadly distributed feedback while remaining sensitive to low-quality behavior.

To verify that this temporal variation arises from intermediate transition scores, we retain the trained model but suppress all nonterminal outputs. The locally flat-step fraction increases from $0.024$ to $0.936$, confirming that the intermediate alignment scores provide the dense signal observed above.

Dream2Reward also exhibits the lowest reward--time dependence, with $I(R;T)=0.011$, compared with $0.083$ for ROBOMETER and $0.042$ for Robo-Dopamine. Progress-based scores may correlate with elapsed time because both tend to increase along nominal demonstrations, whereas Dream2Reward evaluates the transition occurring within each fixed temporal window. Together, the low mutual information and high signal coverage indicate that its dense variation is not primarily explained by trajectory progress.

\begin{table}[tb]
\centering
\caption{Per-step reward signal quality on robo\_arena. Rewards are normalized per method and task using the median and MAD, and results are macro-averaged across tasks. Best values are bold.}
\label{tab:signal_quality}
\setlength{\tabcolsep}{2.2pt}
\begin{tabular}{lccc}
\toprule
Property & D2R & ROBOMETER & Robo-Dopamine \\
\midrule
Locally-flat step fraction $\downarrow$ & \textbf{0.024} & 0.032 & 0.577 \\
Signal coverage $\uparrow$ & \textbf{0.942} & 0.925 & 0.422 \\
Low-qual.\ neg.-reward frac.\ $\uparrow$ & \textbf{0.745} & 0.413 & 0.061 \\
Reward--time MI $I(R;T)$ $\downarrow$ & \textbf{0.011} & 0.083 & 0.042 \\
\bottomrule
\end{tabular}
\end{table}

\subsection{Real-Robot Offline RL with Q-Guided Flow}
\label{sec:real_robot_qgf}

\begin{figure}[tbp]
    \centering
    \subfloat[The four contact-rich tabletop manipulation tasks: cucumber-to-plate, bag opening and package insertion, pressure-cooker placement with vegetables, and divider removal with food extraction.\label{fig:real_robot_tasks}]{%
        \includegraphics[width=\linewidth]{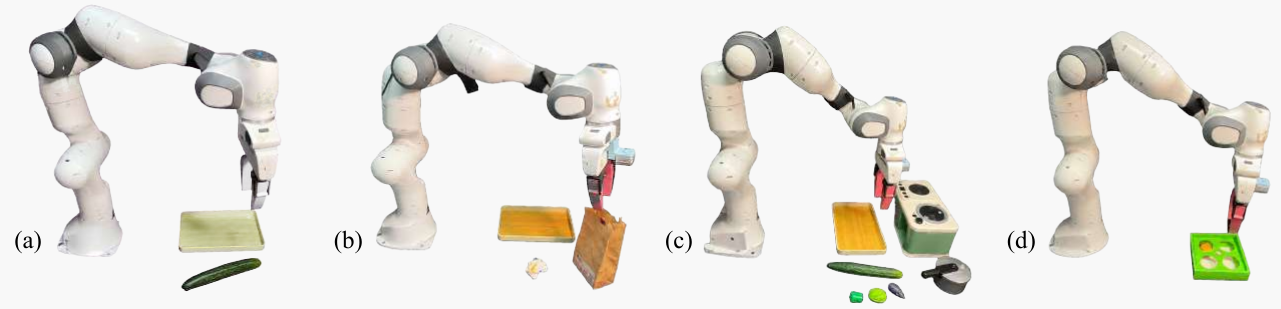}}

    \vspace{0.4em}
    \subfloat[Offline RL success rates per task. Each group compares the same fine-tuned BC policy without guidance against QGF policies trained with dense rewards from Robo-Dopamine, ROBOMETER, and Dream2Reward. Each per-task rate is over 50 evaluation trials; the Average group pools all four tasks ($n=200$), equal to the macro-average because every task has the same number of trials.\label{fig:real_robot_success_rates}]{%
        \includegraphics[width=\linewidth]{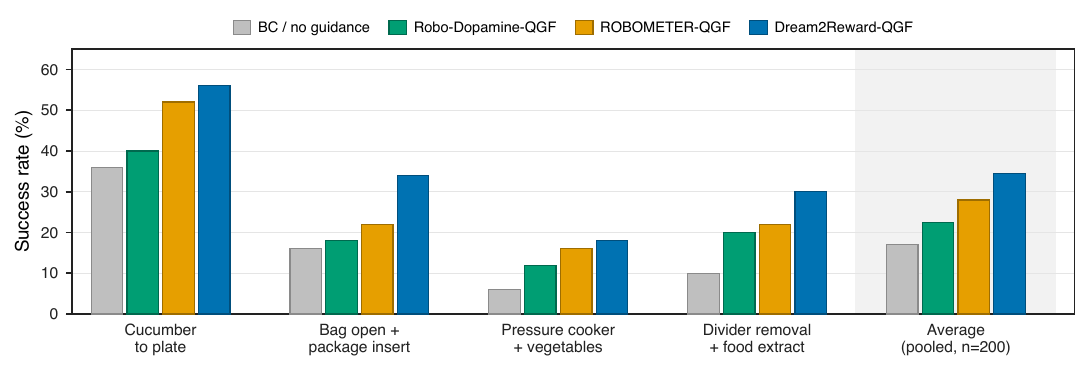}}
    \caption{Real-robot offline RL with Q-guided flow on four tabletop manipulation tasks: the task suite (top; panels (a)--(d) show the four tasks) and per-task offline RL success rates (bottom).}
    \label{fig:real_robot}
\end{figure}

The preceding analyses show that Dream2Reward provides dense, behavior-sensitive feedback with limited dependence on elapsed time. We next test whether these reward properties translate beyond online simulation to zero-shot transition labeling for real-robot offline policy learning.

We evaluate zero-shot reward labeling for real-robot offline RL. For each task, 50 human demonstrations lightly fine-tune a DROID-pretrained $\pi_{0.5}$ policy~\cite{openpi2025}; the resulting policy without Q guidance is the behavior-cloning (BC) baseline. We then run that policy for 50 additional trajectories per task. Unlike the demonstrations, these policy rollouts contain both successful and failed executions without manual dense rewards.

Each frozen reward model labels every transition in the combined demonstration and policy-collected dataset, comprising 100 trajectories per task. For each reward method, the normalization statistics $(\mu,\sigma)$ are estimated from the fixed combined dataset of demonstrations and policy-collected trajectories before critic training and are subsequently held fixed. A separate task critic is then trained on this fixed mixed-quality dataset with the IQL-style objective used by Q-Guided Flow (QGF)~\cite{qgf2026}, which adopts an IQL-style critic objective~\cite{iql2022}. At test time, QGF guides the shared fine-tuned flow policy along the learned critic gradient~\cite{qgf2026}. 

The four contact-rich tabletop tasks are \emph{cucumber to plate}, \emph{bag opening and package insertion}, \emph{pressure cooker and vegetables}, and \emph{divider removal and food extraction} (Fig.~\ref{fig:real_robot_tasks}). They include long-horizon insertion, multi-stage placement, object contact, and recovery, so a single trajectory may contain useful partial motion followed by a stall, drop, or incorrect contact. This mixed structure makes the quality of per-transition critic labels important even when the episode outcome is known only after execution.

Across 50 evaluation trials per task, Dream2Reward-QGF achieves the highest success rate on every task (Fig.~\ref{fig:real_robot_success_rates}): $56\%$ ($28/50$), $34\%$ ($17/50$), $18\%$ ($9/50$), and $30\%$ ($15/50$) on cucumber-to-plate, bag opening and package insertion, pressure cooker and vegetables, and divider removal and food extraction, respectively. Its pooled success is $34.5\%$ ($69/200$), compared with $28\%$ ($56/200$) for ROBOMETER-QGF, $22.5\%$ ($45/200$) for Robo-Dopamine-QGF, and $17\%$ ($34/200$) for BC. The observed pooled success rate for Dream2Reward-QGF is $17.5$ percentage points above BC and $6.5$ points above ROBOMETER-QGF; per-task observed differences versus ROBOMETER-QGF range up to $12$ points and favor Dream2Reward on all four tasks.

\subsection{Transition-Level Credit Assignment in Offline RL}
\label{sec:offline_why}

Because QGF learns from per-transition reward labels, we next examine whether the reward streams provide informative transition-level supervision beyond their final policy outcomes. We use the robo\_arena dataset and the diagnostic normalization and episode-quality groups defined in Sec.~\ref{sec:online_why} to compare the label structure produced by each reward model. Specifically, we evaluate whether the labels provide non-trivial supervision throughout a trajectory, distinguish failed from successful behavior when rewards decline, and reflect execution quality near the end of an episode. To quantify these properties, we define three complementary diagnostics. Among them, the effective-credit fraction and end-segment metrics are first averaged within each task and then macro-averaged across the 40 tasks.

\begin{enumerate}[
    label=(\arabic*),
    leftmargin=*,
    labelsep=0.4em,
    itemsep=0.4em,
    topsep=0.3em,
    parsep=0pt
]
    \item \emph{Effective-credit fraction.}
    A step carries effective credit when its diagnostic reward differs from the trajectory median by more than $0.5$ trajectory MAD, i.e.,
    $|r_t^{\mathrm{diag}}-\operatorname{median}(r^{\mathrm{diag}})|
    \allowbreak > 0.5\,\operatorname{MAD}(r^{\mathrm{diag}})$.
    The resulting fraction measures how broadly the reward provides non-trivial local supervision.

    \item \emph{Failure-drop min-AUC.}
    Within each task, below-median reward steps from successful and failed episodes are pooled to compute a success-versus-failure ROC-AUC. We report the minimum AUC across tasks to capture the weakest task-level failure separation.

    \item \emph{End-segment metrics.}
    Over the final $20\%$ of each trajectory, we compute the success--failure reward gap and the mean reward of failed episodes, measuring terminal separation and failure feedback.
\end{enumerate}

Table~\ref{tab:offline_credit} shows that Dream2Reward supplies non-trivial critic labels on $80.1\%$ of steps, compared with $70.3\%$ for ROBOMETER and $40.4\%$ for Robo-Dopamine. It also leads the failure-drop separation and end-segment gap, and gives the most negative mean reward on failed end segments ($-0.091$, versus $-0.001$ and $+0.004$). For critics trained from per-transition rewards, broad label coverage and low failed-ending values are desirable; the robo\_arena results show both properties, while the separate four-task experiment measures downstream QGF success.

\begin{table}[tb]
\centering
\caption{Offline-relevant transition-label statistics on robo\_arena. Rewards are normalized per method and task using the median and MAD, and results are macro-averaged over 40 tasks. Best values are bold.}
\label{tab:offline_credit}
\setlength{\tabcolsep}{2.2pt}
\begin{tabular}{lccc}
\toprule
Property & D2R & ROBOMETER & Robo-Dopamine \\
\midrule
Effective-credit step fraction $\uparrow$ & \textbf{0.801} & 0.703 & 0.404 \\
Failure-drop min-AUC $\uparrow$ & \textbf{0.661} & 0.580 & 0.333 \\
End-segment success--fail gap $\uparrow$ & \textbf{0.098} & 0.056 & 0.009 \\
Failure end-segment reward $\downarrow$ & \textbf{$-0.091$} & $-0.001$ & $+0.004$ \\
\bottomrule
\end{tabular}
\end{table}

Across robo\_arena, Dream2Reward's stream is broadly varying, weakly dependent on elapsed time, and lower on transitions drawn from failed or low-quality episodes. Separate LIBERO and real-robot QGF experiments show higher observed policy success under their respective protocols, consistent with the reward properties measured in the diagnostics.

\subsection{Summary and Discussion}

The experiments provide complementary evidence at the transition, reward-stream, and policy-learning levels. The alignment diagnostics show that Dream2Reward responds to task-mismatched and failed transitions, while the trajectory-level analyses demonstrate consistent separation between successful and failed execution. The independent robo\_arena evaluation further shows that its rewards remain informative over meaningful portions of a trajectory, provide negative feedback for low-quality behavior, and are less dominated by elapsed time. Finally, the LIBERO and real-robot QGF experiments show that the same frozen reward model supports stronger online and offline policy learning. Across these independent datasets and evaluation protocols, the results consistently support local transition correctness as useful supervision for reward-guided optimization.

Transition alignment also distinguishes Dream2Reward from progress- and generation-based reward readouts. Scalar progress orders observations along nominal successful trajectories, but does not directly assess whether the realized displacement has the appropriate direction and magnitude. Likelihood, prediction residual, and denoising energy instead measure how well a generative model explains an endpoint or conditional future, and may therefore reflect appearance uncertainty or model fit without explicitly evaluating motion correctness. Dream2Reward uses the predicted successful continuation to construct a reference displacement and directly compares it with the observed displacement. The signed cosine term captures whether the realized motion follows or opposes the task-conditioned direction, while the symmetric log-ratio penalizes proportional under- and over-motion. Their product therefore requires agreement in both direction and magnitude. A fixed inference horizon provides a consistent temporal scale, while conditioning only on observations available at the transition start preserves the causal construction of the reference. Transition alignment thus complements global progress estimation with a local measure of whether each realized transition is consistent with successful task execution.


\section{Conclusion}
\label{sec:conclusion}

Dream2Reward learns dense, task-conditioned rewards from successful demonstrations by predicting a successful latent displacement from past visual context and language, and then evaluating each observed transition through directional and magnitude alignment. This construction preserves causal scoring and provides a local measure of transition correctness rather than relying solely on trajectory progress or generative fit. Our analysis establishes bounded sensitivity to perturbations in the predicted transition field under non-degenerate conditions. Across transition-level diagnostics, reward-stream analyses, online reinforcement learning in LIBERO, and real-robot offline policy learning, the same frozen reward model consistently separates successful from failed behavior, provides informative per-step feedback, and supports stronger downstream policy learning than the compared reward models. These results support transition alignment as an effective approach for converting positive demonstrations into dense rewards for online and offline robot learning.



\bibliographystyle{IEEEtran}
\bibliography{example}

\end{document}